\documentclass[letterpaper]{article} % DO NOT CHANGE THIS
\usepackage{aaai23}  % DO NOT CHANGE THIS
\usepackage{times}  % DO NOT CHANGE THIS
\usepackage{helvet}  % DO NOT CHANGE THIS
\usepackage{courier}  % DO NOT CHANGE THIS
\usepackage[hyphens]{url}  % DO NOT CHANGE THIS
\usepackage{graphicx} % DO NOT CHANGE THIS
\usepackage{natbib}  % DO NOT CHANGE THIS AND DO NOT ADD ANY OPTIONS TO IT
\usepackage{caption} % DO NOT CHANGE THIS AND DO NOT ADD ANY OPTIONS TO IT
\usepackage{subcaption}
\usepackage{color}
\usepackage{balance}
\usepackage{amsmath}
\DeclareCaptionStyle{ruled}{labelfont=normalfont,labelsep=colon,strut=off} % DO NOT CHANGE THIS

\usepackage{algorithm}
\usepackage[noend]{algpseudocode}
\usepackage{amsthm, amssymb}
\newtheorem{theorem}{Theorem}
\newtheorem{definition}{Definition}
\newtheorem{corollary}{Corollary}

\newtheorem{prop}{Proposition}
\usepackage{bibentry} % for list of publications
\definecolor{green(pigment)}{rgb}{0.1607, 0.3843, 0.0941}
\definecolor{blue(pigment)}{rgb}{0., 0.1484, 0.6992}
\definecolor{orange(pigment)}{rgb}{0.6, 0.298, 0.}
\definecolor{blue(back)}{rgb}{0.9058, 0.9019, 0.9725}

\usepackage{tcolorbox}
\usepackage{newfloat}
\usepackage{listings}
\floatstyle{ruled}
\newfloat{listing}{tb}{lst}{}
\floatname{listing}{Listing}

\title{Certified Policy Smoothing for Cooperative Multi-Agent Reinforcement Learning}

\author{
    Ronghui Mu\textsuperscript{\rm 1}\footnote{Ronghui conducted this research when she was a visiting PhD student at the University of Exeter.},
    Wenjie Ruan\textsuperscript{\rm 2}$^{\dagger}$,
    Leandro Soriano Marcolino\textsuperscript{\rm 1},
    Gaojie Jin\textsuperscript{\rm 3},
    Qiang Ni\textsuperscript{\rm 1}
}

\affiliations {
    \textsuperscript{\rm 1}{School of Computing \& Communication, Lancaster University, Lancaster, LA1 4YW, UK}\\
    \textsuperscript{\rm 2} {Department of Computer Science, University of Exeter, Exeter, EX4 4QF, UK}\\
\textsuperscript{\rm 3} {Department of Computer Science, University of Liverpool, Liverpool, L69 3BX, UK}\\
    \{r.mu, l.marcolino, n.qiang\}@lancaster.ac.uk, w.ruan@exeter.ac.uk, g.jin3@liverpool.ac.uk\\
        $\dagger~$Corresponding Author
}

\usepackage{bibentry}
\begin{document}

\maketitle

\begin{abstract}
Cooperative multi-agent reinforcement learning (c-MARL) is widely applied in safety-critical scenarios, thus the analysis of robustness for c-MARL models is profoundly important. However, robustness certification for c-MARLs has not yet been explored in the community. In this paper, we propose a novel certification method, which is the first work to leverage a scalable approach for c-MARLs to determine actions with guaranteed certified bounds. c-MARL certification poses two key challenges compared with single-agent systems: (i) the accumulated uncertainty as the number of agents increases; (ii) the potential lack of impact when changing the action of a single agent into a global team reward. These challenges prevent us from directly using existing algorithms. Hence, we employ the false discovery rate (FDR) controlling procedure considering the importance of each agent to certify per-state robustness and propose a tree-search-based algorithm to find a lower bound of the global reward under the minimal certified perturbation. As our method is general, it can also be applied in single-agent environments. We empirically show that our certification bounds are much tighter than those of the state-of-the-art RL certification solutions. We also run experiments on two popular c-MARL algorithms: QMIX and VDN, in two different environments, with two and four agents. The experimental results show that our method produces a meaningful guaranteed robustness for all models and environments. Our tool \textbf{CertifyCMARL} is available at \url{https://github.com/TrustAI/CertifyCMARL}.         
\end{abstract}

\section{Introduction}

 Recently, cooperative multi-agent reinforcement learning (c-MARL) has attracted increasing attention from researchers and is beneficial for a wide range of applications in the real world, such as autonomous cars \cite{shalev2016safe}, traffic lights control \cite{van2016coordinated}, packet delivery \cite{ye2015multi} and wireless communication \cite{station}. As it is widely involved in safety-critical scenarios, there is an urgent need to analyze the robustness of c-MARLs.
 %On the other hand, extensive studies have shown that DNNs are vulnerable to tiny non-random, ideally, human-invisible perturbations of the input to lead to their abominable predictions \cite{szegedy,carlini2017evaluating}, which delivers an urgent need to analyze the robustness of c-MARL. 
 
 Reinforcement learning (RL) aims to find the best actions for agents that can optimise the long-term reward by interacting with the surrounding environments \cite{qmix,cai2023reprem}. When there is a team of agents, the system needs to jointly optimise the action of each agent to maximise the reward of the team. In c-MARL, as the number of agents increases, the joint action space of the agents grows exponentially, requiring the learning of policies in a decentralised manner \cite{oliehoek2008optimal}. In this approach, each agent learns its own policy, based on its local action-observation history, and then forms a centralised action-value that is conditioned to the global state and joint actions.
 
Deep neural networks (DNNs) are known to be vulnerable to tiny, non-random, ideally human-invisible perturbations of the input, which can lead to incorrect predictions \cite{szegedy,xu2022quantifying,jin2022enhancing,wang2022deep,mu20223dverifier,yin2022dimba,ruan2019global}. RL has also been shown to be susceptible to perturbation in the observations of an RL agent \cite{huang2017adversarial,behzadan2017vulnerability} or in environments \cite{gleave2019adversarial}. Some adversarial defence works for RL are proposed \cite{donti2020enforcing,eysenbach2021maximum,shen2020deep,defense} and then towards these defences, stronger attacks are proposed \cite{salman2019provably,russo2019optimal}. To end this repeated game, \citet{wu2021crop} and \citet{kumar2021policy} proposed to use probabilistic approaches to provide robustness certification for RLs. Concerning c-MARL, \citet{lin2020robustness} addressed the challenges of attacking such systems and proposed adding perturbations to the state space. To date, the robustness certification on c-MARL has not been touched upon by the community.

Compared to the RL system with a single agent, certifying c-MARL is a more challenging task. \textbf{Challenge 1}: the action space grows exponentially with the number of agents; moreover, for each time step, the agents need to be certified simultaneously, accumulating uncertainty. \textbf{Challenge 2}: changing the action of one agent may not alter the team reward, thus, instead of following existing certification works on a single agent, new criteria should be raised to evaluate the robustness for the multi-agent system. Therefore, to cope with such challenges, we propose two novel methods to certify the robustness of each state and of the whole trajectory. 
 
We first propose a smoothed policy where each agent chooses the most frequent action when its observation is perturbed, and then we derive the certified bound of perturbation for each agent per step, within which the chosen action of the agent will not be altered. When evaluating the robustness of all agents per time step, to tackle \textbf{Challenge 1}, we identify the multiple test problem and propose to {\it correct} the p-value by multiplying the importance factor of each agent. We then employ the Benjamini-Hochberg (BH) procedure with corrected p-value to control the selective false discovery rate (FDR). For the certification of robustness of the global reward, to deal with \textbf{Challenge 2}, we propose a tree-search-based algorithm to find the certified lower bound of the perturbation and the lower bound of the global reward of the team under this perturbation. In this paper, we focus on certifying the robustness of value-based c-MARLs under a $l_2$ norm bounded attack. Our work can be easily extended to evaluate $l_p$ norm based robustness by using different sampling distributions, such as the generalised Gaussian distribution as indicated in \citet{hayes2020extensions}.  
 
 Our main \textbf{contributions} can be summarised as: {\bf i)} for the first time, we propose a solution to certify the robustness of c-MARLs, which is a {\it general} framework that can also be employed in a single-agent system; {\bf ii)} we propose a new criterion to enable the {\it scalable} robustness certification per state for c-MARLs by considering the importance of each agent to reduce the error of selective multiple tests; and {\bf iii)} we propose a tree-search-based method to obtain the certified lower bound of the global team reward, which enables a tighter certification bound than the state-of-the-art certification methods. 
 
 %\begin{itemize}
     %\item We are the first work to provide certification robustness for c-MARLs and our method is general that can also be employed in single agent system. 
     %\item We propose a new criteria to provide salable robustness certification per-state for c-MARLs by asserting the importance of each agent to control the multiple tests error.
     %\item We proposed a search-tree based method to obtain the certified lower bound of global team reward, which can provide tighter certification bound than the state-of-the-art certification method in the single agent environment. 
 %\end{itemize}
 
%%%%%----

\section{Background}

\subsection{Cooperative Multi-agent Reinforcement Learning}
Most c-MARL methods use the centralised training scheme to guide decentralised execution, such as value decomposition networks (VDN) \cite{vdn} and QMIX \cite{qmix}. In this paper, we focus on certifying the robustness of these value-based c-MARLs.

We consider a fully cooperative multi-agent game $G$ as a Dec-POMDP \cite{kraemer2016multi}, which is defined by the tuple $G=\left\langle S,\mathcal{A},P,r,Z,\mathcal{O},N,\gamma\right\rangle$, in which each agent $n \in \{1,2,...,N\}$ chooses an action $a^n \in \mathcal{A}$ in each state $s \in S$ to form the joint action $\mathbf{a}=\{a^1,a^2,...,a^N\}$. The same reward function is shared by all agents $r(s,\mathbf{a})$. $\gamma$ is a discount factor. We suppose that each agent draws an observation $z^n \in Z$ given the observation function $\mathcal{O}(s,\mathbf{a})$. 

Each agent has a stochastic policy $\pi^n(a^n|h^n)$ where $h^n$ is the action-observation history $h^n  \in \mathcal{H}$. The joint policy $\pi$ has a joint discount return $R_t=\sum^{\infty}_{i=0}(\gamma^ir_{t+i})$ and an action-value
function: $Q^{\pi}(s_t,\mathbf{a}_t)=\mathbb{E}_{s_{t+1}:\infty,\mathbf{a}_{t+1}:\infty}[R_t|s_t,\mathbf{a}_t]$. Given an action-value function $Q^{\pi}$, we define a greedy policy as $\pi(s_t):\mathop{\arg\max}_{\mathbf{a}_t\in \mathcal{A}}Q^{\pi}(s_t,\mathbf{a}_t)$ that returns the optimal action.

\subsection{Randomized Smoothing for Classification}

Randomised smoothing \citep{cohen2019certified} was developed to evaluate probabilistic certified robustness for classification tasks. It aims to construct a smoothed model $g(x)$, which can produce the most probable prediction of the base classifier $f(x)$ over perturbed inputs from Gaussian noise in a test instance. The smoothed classifier $g(x)$ is supposed to be provably robust to $l_2$-norm bounded perturbations within a certain radius:
%\begin{tcolorbox}[colback=blue(back), colframe=blue(back)]
\begin{theorem}\label{the:cohen}
 \cite{cohen2019certified} For a classifier $f: \mathbb{R} \to \mathcal{Y}$, suppose $c \in \mathcal{Y}$, let $\delta \sim \mathcal{N}(0, \sigma^2 I)$, the smoothed classifier be 
 $g(x):=\mathop{\arg\max}\limits_c \mathbb{P}(f(x+\delta)=c)$,
 suppose $\underline{p_A},\overline{p_B} \in [0,1]$, if
 \begin{equation}
     \mathbb{P} (f(x+\delta)=c_A)\geq \underline{p_A} \geq \overline{p_B} \geq \mathop{\max}\limits_{c\neq c_A}\mathbb{P}(f(x+\delta)=c),
 \end{equation}
 then $g(x+\epsilon)=c_A$ for all $||\epsilon ||_2 \leq R$, where
 \begin{equation}
     R=\frac{\sigma}{2}(\Phi^{-1}(\underline{p_A})-\Phi^{-1}(\overline{p_B})).
 \end{equation}
\end{theorem}
Here $\Phi^{-1}$ is the inverse cumulative distribution function (CDF) of the normal distribution.

\section{Policy Smoothing for c-MARLs}

In this section, we first outline an intuitive approach to certifying RL based on current classifier certification. Then, we sketch the challenges preventing the direct use of the intuitive approach and present how to address these challenges. 

\subsection{Problem Formulation}
We aim to design a robust policy for multi-agent reinforcement learning algorithms. Following the standard setting of existing adversarial attacks on c-MARLs, e.g. \cite{lin2020robustness}, where the adversarial perturbation is added to each step's observation of each agent, our proposed policy is expected to be provably robust against the perturbation bounded by the $l_2$-norm around the observation of each agent.

\begin{definition}\label{def:1}
(Smoothed policy) Given a trained multi-agent reinforcement learning network $Q^{\pi}$ with policy $\pi$, suppose that there are N agents, at the time step $t$, let $\forall s_{t} \in S$, given that the noise vector $\Delta_t=(\delta^1_t,...,\delta^N_t)$ is $i.i.d \quad \mathcal{N}(0, \sigma^2 I)$, the joint smoothed policy can be represented as    
\begin{equation}\label{eq:1}
\begin{aligned}
\tilde{\pi}\left(s_{t}\right)&=\mathop{\arg\max}\limits_{\mathbf{a}_t \in \mathcal{A}} \widetilde{Q}^{\pi}\left(s_{t} + \Delta_t , \mathbf{a}_t\right) \\
\end{aligned}
\end{equation}
\end{definition}

To certify the robustness of the smoothed policy, we define the certification robustness for a per-step action as 
\begin{equation}
  \begin{aligned}
   \tilde{\pi}_t(s_t)=\tilde{\pi}_t(s_t+\epsilon_t ) \quad s.t. \forall \epsilon_t, ||\epsilon_t||_2 \leq D
  \end{aligned}
\end{equation}
where $\epsilon_t \in \mathbb{R}^{ N}$ represents the maximum perturbation applied to the observations of each agent at the $t$-th time step. In other words, for each agent, in the presence of the $l_2$-norm bounded perturbation in each state, the smoothed policy is expected to return the same action that is most likely to be selected in the unperturbed state $s_t$.
\subsection{Intuitive Approach}
\begin{algorithm}[t]
\caption{Intuitive Policy Smoothing for Certifying \\Per-state Action}
\label{alg:algorithm1}
\textbf{Input}: Trained $Q^{\pi}$  with $N$ agents\\
\textbf{Parameter}: sampling times $M$; Gaussian distribution parameter $\sigma$;  confidence parameter $\alpha$ \\
\begin{algorithmic}[1]
\Function{smoothing}{$M, Q,{\alpha}, {\sigma}$}
\For {$m \gets 1,M$} \Comment {Get smoothed policy $\tilde{\pi}$}
    \State generate $\Delta_m=(\delta^1_m,...,\delta^N_m)$  i.i.d $\mathcal{N}(0, \sigma^2 I)$ 
    \State $s' \gets s + \Delta_m$ 
    \State $\mathbf{a} \gets \pi(s')$
    \State Add $\mathbf{a} \to Actlist$
    \EndFor
\State \textbf{return} $Actlist$
\EndFunction
\State $Actlist \gets \textproc{smoothing}(M, Q^{\pi},{\alpha}, {\sigma})$
\State $\mathbf{a}^m, \mathbf{a}^r,ct_1, ct_2  \gets$ Top two action sets with their counts
\If {$BioPVALUE(ct_1,ct_1+ct_2,0.5)\leq \alpha$}
\State $Cert \gets True$ \Comment{Get certified radius for $\tilde{\pi}$}
\State $\underline{p_\mathbf{a^m}},\overline{p_{\mathbf{a}^r}} \gets MultiConBnd(Counts(Actlist),\alpha)$
\State $D \gets \frac{\sigma}{2}\left(\Phi^{-1}\left({\underline{p_{\mathbf{a}^m}}}\right)-\Phi^{-1}\left(\overline{p_{\mathbf{a}^r}}\right)\right)$
\Else
\State $Cert \gets False$, $D \gets 0$
\EndIf
\State \textbf{return} $\mathbf{a}, d, Cert$
\end{algorithmic}
\end{algorithm}

Intuitively, the randomised smoothing can be adapted to certify the robustness of the per-state action in RLs by replacing the classifier $f(x)$ with policy $\pi(s_t)$. For the certification of each step, Monte Carlo randomised sampling is used to estimate the smoothed policy $\tilde{\pi}$. As shown in Algorithm \ref{alg:algorithm1}, we record the action vector $\mathbf{a}$, which is a combination of actions taken by all agents at each sampling step. The most likely selected action set is chosen as the action taken by $\tilde{\pi}$. A larger number of samples can be used to estimate the lower bound on the probability ($\underline{p_{\mathbf{a}^m}}$) of the most frequently selected action set, $\mathbf{a}^m$, and the upper bound on the probability ($\overline{p_{\mathbf{a}^r}}$) of the second most frequently selected (``runner-up") action, $\mathbf{a}^r$. The function \textproc{MultiConBnd} in Algorithm \ref{alg:algorithm1} is based on a Chi-Square approximation \cite{goodman1965simultaneous}, which takes the number of observations for each category as input and returns the $(1-\alpha)$ confidence levels. 
\begin{prop}
If the certification in Algorithm 1 returns the action set $\mathbf{a}^m:(a^{m,1},a^{m,2},...,a^{m,N})$ in the time step $t$ with a certified radius $D=\frac{\sigma}{2}\left(\Phi^{-1}\left(\underline{p_{\mathbf{a}^m}}\right)-\Phi^{-1}\left(\overline{p_{\mathbf{a}^r}}\right)\right)$
then with probability at least $(1-\alpha)$, the smoothed policy $\tilde{\pi}(s_{t}+\epsilon_t)$ chooses the action $\mathbf{a}^m$, $\forall ||\epsilon_t||_2 \leq D$.
\end{prop}

%The certified radius D is estimated based on the probability calculated from \textproc{MultiConBnd}. Thus, with probability at least $1-\alpha$, it follows the \textbf{Theorem}.
Proof is provided in {\bf Appendix\footnote{All appendixes of this paper can be found at \url{https://github.com/TrustAI/CertifyCMARL/blob/main/appendix.pdf}} A}. The intuition behind the method shown in Algorithm \ref{alg:algorithm1} is similar to the certification procedure for classification through randomised smoothing \cite{cohen2019certified}. The \textproc{BioPValue} is applied to calculate the p-value of the two-sided hypothesis test to choose the action $\mathbf{a}^m$. However, rather than abstaining from the action when the p-value does not meet the confidence level, we set the certified radius of this step as $D=0$ to indicate that the certification failed, since the RL relies on decisions of multiple steps. When $n=1$, the algorithm can be used to certify RLs with a single agent as \citet{wu2021crop}, but instead of using the smoothed action value function $Q^{\pi}$, we utilise the frequency of occurrence of each action to determine which action to be selected. Since c-MARLs are trained under the premise that each agent would always select the best action, they do not reliably anticipate the team reward when some agents behave badly.
%\subsection{Challenges for Certifying c-MARLs}\label{challenges}

In the c-MARLs, there are some additional challenges that preclude us from using this intuitive certification criterion. 

\textbf{Challenge 1. The perturbation $D$ added to the observation of each agent can be different.}
For c-MARLs, each agent develops its own policy to choose its action. If the certified bound is calculated using Algorithm \ref{alg:algorithm1}, all agents will engage with the same perturbation bound, making the results less accurate for each agent. As one agent can be more robust than the other, the same perturbation added to the agents will lead to different performances, which provides the need to certify the robustness of each agent. Thus, we will first consider certifying the robustness for every agent and then estimating the robustness in each state for all agents. To reduce the computation cost, we can sample from the joint policy $\pi(s')$ instead of each agent's policy separately. To this end, we can change $\mathbf{a}^m, \mathbf{a}^r$ in Algorithm \ref{alg:algorithm1} (Line 9) to the two most likely actions $(a^{m,n}, a^{r,n})$ for each agent and then calculate the corresponding lower bound $\underline{p_{a^{m,n}}}$ on probability $\mathbb{P}(\tilde{\pi}^n(z^n):=a^{m,n})$ and upper bound $\overline{p_{a^{r,n}}}$ for choosing the ``runner up" action, $a^{r,n}$. The certified bound for each agent per state can be computed as:

\begin{corollary}\label{theo:radius}
(Certification for the actions of each agent in each state) In state $s$, given the joint smoothed policy $\tilde{\pi}(s) =\{\tilde{\pi}^1(z^1),...,\tilde{\pi}^N(z^N)\}$, we can obtain the certified bound in state $s$ for each agent to guarantee $\tilde{\pi}^n(z^n+\epsilon):=a^{m,n}, \forall ||\epsilon ||_2 \leq d_n$:
 \begin{equation}
     d_n=\frac{\sigma}{2}(\Phi^{-1}(\underline{p_{a^{m,n}}})-\Phi^{-1}(\overline{p_{a^{r,n}}}))
 \end{equation}
\end{corollary}
Proof is provided in {\bf Appendix A}. Finally, the most likely chosen action for each agent can be combined as the final action set $\{{a}^{m,1},{a}^{m,2},...,a^{m,N}\}$ and the certified bound at each step can be defined as: 
 \begin{definition}\label{def:2}
Given the certified bound obtained for each agent in state $s$, $\{d_1,d_2,...,d_N\}$, the certified bound in this state for all agents is determined by the least robust agent: $D = min\{d_1,d_2,...,d_N\}$.
\end{definition}

\textbf{Challenge 2. If we choose the bound of the least robust agent as the bound for all agents per state, the confidence level decays.}
As Proposition 1 indicates, on each call of certification, the certified robustness bound obtained only holds with confidence level $(1-\alpha)$. As we sample noise from the Gaussian distribution \textit{independently}, the hypothesis tests are \textit{independent}. Based on Definition \ref{def:2}, to calculate the certified bound for each state, we have the following constraint for the probability of making an error:
$$\begin{array}{l}\mathbb{P}(\bigvee_{n \in N}, n \text {-th agent's cert failed }) \leq \\ \min \left(\sum_{n} \mathbb{P}(n \text {-th agent's cert failed }), 1\right)=\min (N \alpha, 1)\end{array}
$$
%$\mathbb{P}\left(\bigvee{n}, n \text {-th agent's cert failed }\right) \leq \\
%\min \left(\sum_{n} \mathbb{P}(n \text {-th agent's cert failed}), 1\right)
%=\min (N \alpha, 1)$
Therefore, for multiple tests, without any control on the error, the probability of making an error will increase with the number of tests. Suppose that there are $T$ steps in the entire trajectory, we will have $N*T$ tests in total, which can be a great challenge. To address this problem, for certifying per-state actions, the confidence level can be reduced to $\alpha / N$. Additionally, we can first perform agent selection to control the selective error by considering the importance of each agent, since sometimes an agent changing its action will not diminish the team reward.
%Figure \ref{Fig:heatmap} shows an example of each step's team reward when one agent's action changes and the other remain the same. We can see that for the trained VDN model if the action selected for agent one is not intervened, altering the action of agent two will not lead to the change of per-step team reward. 
Moreover, to evaluate the global certification bound, we propose a tree-search-based method to find the lower bound of the team reward. In Section \ref{methodology} and \ref{globalreward}, we will detail our proposal to certify the robustness of per-state actions and global reward.

\section{Robustness Certification for Per-State Action with Correction}\label{methodology}
%Existing work on certifying the robustness of reinforcement learning models are all designed to measure the guarantees in $l_2$ norm, which required to sampling from the Gaussian distribution for the smoothing measure. \citet{hayes2020extensions} proposed that if used the generalized Gaussian distribution \cite{nadarajah2005generalized} to replace the standard Gaussian distribution, then we can compute the robustness guarantee for $l_p$ norm perturbation. The density function of $\mathcal{G} \mathcal{N}(\mu, \sigma, p)$ can be expressed as:
%\begin{equation}\label{eq:gn}
%p(x)=\frac{s}{2 \sigma \Gamma\left(\frac{1}{s}\right)} e^{-\left|\frac{x-\mu}{\sigma}\right|^{p}}
%\end{equation}
%\begin{theorem}\label{theo:1}
%Let $\underline{p^n_\mathcal{A}}$ be the lower bound of the probability returned by the most probable action for agent n and $\overline{p^n_\mathcal{A'}}$ is the upper bound returned by the "runner-up" action. The $l_p$ certified radius is  
%$r = -\log \left(2 \sqrt{p_{1} p_{2}}+1-p_{1}-p_{2}\right)$, where $p \in \{1,2\}$.
%\end{theorem}
\subsection{Multiple Hypothesis Testing }
\begin{corollary}\label{cor:bound}
(Certified bound per state) In state $s$, given N agents with action $\mathbf{a}$, the joint policy is ${\pi}(s) =\{{\pi}^1(z^1),...,{\pi}^N(z^N)\}$. Suppose that the observation of each agent is perturbed by random noise $\delta^n$, where $\delta^n \sim \mathcal{N}(0, \sigma^2 I)$. $\forall n \in N$, if $\mathbb{P}({\pi}^n(z^n+\delta^n):=a^{m,n})\geq 0.5$, we can compute the certified bound by \textbf{Definition \ref{def:2}}.
\end{corollary}
Proof is presented in {\bf Appendix B}. In order to obtain the certified bound for all agents per state, we can employ Corollary \ref{cor:bound}, and, as suggested, for each agent, we need to ensure that condition $\mathbb{P}(\tilde{\pi}^n(z^n):=a^{m,n})\geq 0.5$ is satisfied. Hence, after sampling, with the count ($ct^n_1$) for the most frequent action taken by agent $n$, we can implement the one-sided binomial test to obtain its p-value $pv_n$. These p-values can be processed to indicate which tests should be accepted under $(1-\alpha)$ confidence.

\begin{definition}\label{def:4}
(Hypothesis Test) The hypothesis test with null hypothesis for each agent is $H_0: \mathbb{P}(\tilde{\pi}^n(z^n):=a^{m,n}) < 0.5$, and the alternative is $H_1: \mathbb{P}(\tilde{\pi}^n(z^n):=a^{m,n}) \geq 0.5$ 
\end{definition}
In the hypothesis test, if the null hypothesis $H_0$ is true, we can determine the p-value, which is the probability of finding a statistic that is equally extreme as the observed one or more extremes. Given the statistical test in \textbf{Definition \ref{def:4}}, if the p-value is below the confidence level, we can reject the null hypothesis, which means that the bound is certified; otherwise, we accept it.   

In multiple hypothesis tests, the probability of the occurrence of false positives (FP) will increase, where the FP denotes that we reject the null hypothesis when it is true, which is also called {\it type I error}. Suppose that the confidence level is $\alpha$, the probability of FP is expected to be less than $\alpha$. To control {\it type I error} for multiple tests with $H$ tests, the family-wise error rate (FWER) is introduced, which changes $\alpha$ for each test to $\alpha/H$. However, it is still conservative, which can increase the true negative rate (i.e., {\it type II error}). 

%%%%%%%%%%%%visit-1%%%%%%%%%%%%%%%%

To solve this problem, \citet{benjamini1995controlling} proposed the false discovery rate (FDR) to find the expected false positive portion. The FDR method applies a corrected p-value for each test case, achieving a better result: testing for as many positive results as possible while keeping the false discovery rate within an acceptable range. The Benjamini-Hochberg (BH) procedure first sorts the p-values of tests in ascending order and then finds the largest $k$ such that $p_k \leq k \alpha/H$, rejecting null if the p-value is below $p_k$. \citet{Fit2014} then proposed selective hypothesis tests by applying inference to the selected model to control the selective {\it type I error}, which controls the global error as $\frac{\mathbb{E} [\# False Rejections]}{\mathbb{E} [\# H_0 Selected]} \leq \alpha$.
Inspired by the selective hypothesis tests, we propose to multiply every agent's importance factor with its p-value to control the selective FDR via
executing the BH procedure on the corrected p-values.

\begin{algorithm}[tb]
\caption{Certified Robustness Bound of the Perturbation for Actions of Each State with Correction (CRSC)}
\label{alg:algorithm2}
\textbf{Input}: Trained $Q^{\pi}$; $N$ agents; \\
\textbf{Parameter}: sampling size $M$; Gaussian distribution parameter $\sigma$;  confidence parameter $\alpha$   
\begin{algorithmic}[1] 
\State $Actlists \gets \textproc{smoothing}(M, Q^{\pi},{\alpha}, {\sigma})$
\State $a^{m,n}, a^{r,n},ct^n_1, ct^n_2  \gets Counts(Actlists[n])$ for $n \in N$
\State $IF \gets IF\_function(Q^{\pi},Actlists)$ 
\Statex \Comment {Obtain importance factor for agent}
\State $pv_n \gets BioPVALUE(ct^n_1,M,0.5)$ for $n \in N$
\State $c_n \gets BHproc((pv_n*IF[n]),\alpha)$ for $n \in N$
\State If $\neg c_{n}: d_{n} \leftarrow 0$ \Comment {Remove failed agent}
\State $\mathcal{I}_{\boldsymbol{cert}}:=\left\{n \mid {d}_{n} \neq 0 \right\}$ \Comment {Obtain certified agent set}
\State Compute $d_n$ for each agent in $\mathcal{I}_{\boldsymbol{cert}}$
\State $D =min(d_n|n\in\mathcal{I}_{\boldsymbol{cert}})$ 
\State \textbf{return} D, $\mathcal{I}_{\boldsymbol{cert}}$ 
\end{algorithmic}
\end{algorithm}
\begin{algorithm}[t]
\caption{Tree-Search-based certified robustness bound and global reward (T-CRGR)}
\label{alg:algorithm3}
\textbf{Input}: Trained $Q^{\pi}$; $N$ agents; confidence parameter $\alpha$\\
\textbf{Parameter}: sampling times $M$; Gaussian distribution parameter $\sigma$ \\
\begin{algorithmic}[1]
\Function{GetNode}{$s$}
\State $Actlists \gets \textproc{smoothing}(M, Q^{\pi},{\alpha}, {\sigma})$
\State $IF \gets IF\_function(Q,Actlists) $ 
\State $A\_dic,d\_list \gets \emptyset$
\For{$ n \in \mathcal{I}_{\boldsymbol{agent}}$}
    \State $a^{m,n}, a^{r,n},ct^n_1, ct^n_2  \gets$ $Counts(Actlists)$
    \State $pv_n \gets BioPVALUE(ct^n_1,M,0.5)$
    \If {$pv_n*IF[n] > \alpha$}
    \State $A\_dic[n] \gets A\_dic[n]\cup\{a^{m,n}, a^{r,n}\}$
    \State $ \underline{p_1}\gets BioConBnd(ct^n_1+ct^n_2,M,1-\alpha)$
    \Else 
    \State $A\_dic[n] \gets A\_dic[n]\cup\{a^{m,n}\}$
    \State $ \underline{p_1}\gets BioConBnd(ct^n_1,M,1-\alpha)$
    \EndIf
    \State $d\_list \gets d\_list \cup (\sigma\Phi^{-1}\left({\underline{p_1}}\right))$
\EndFor
\State $d \gets min(d\_list)$
\State \textbf{return} $A\_dic,d$
\EndFunction
\Function{Search}{$d,s,\mathbf{a},R,done$}:
\If {$R\geq R_{min}$} \Comment{{Prune the tree}}
\State \textbf{return} 0
\EndIf
\If {$done$}
\State $R_{min} \gets min(R,R_{min})$
\State \textbf{return} 0
\EndIf
\State $A\_dic,d_{new} \gets \textproc{GetNode}(s)$
\State $d\gets min(d_{new},d), Action\_list \gets A\_dic$
\For {$\mathbf{a}$ in $Action\_list$}
\State $s_{new},done \gets env.step(\mathbf{a},s) $
\State $\textproc{Search}(d,s_{new},\mathbf{a},R+Q(s,\mathbf{a}),done)$
\EndFor
\EndFunction
\end{algorithmic}
\end{algorithm}
%%%%%%%%%%%%visit-2%%%%%%%%%%%%%%%%
\subsection{Measuring the Importance of Agents}

To obtain each agent's important factor, we can measure each agent's contribution to the team reward at each state. We adapt the advantage function proposed in COMA \cite{coma}, which is used to decentralise agents by estimating the individual reward during training. As the importance factor defined in \textbf{Definition} \ref{def:5}, it is applied to examine the behaviour of the current action of the agent.   
\begin{definition}\label{def:5}
For each agent $n$, the importance factor $IF^n$ of each agent is computed by comparing the Q value of the current action $a^n$ with the counterfactual reward baseline, which is obtained by altering the action of agent $n$, $a^{n'}$, and keeping the other agents’ actions $\mathbf{a}^{-n}$ unchanged:
\small{
$$IF^{n}(s, \mathbf{a})=Q(s, \mathbf{a})-\sum_{a^{n'} \in \mathcal{A}} \mathbb{P}(\tilde{\pi}^{n}(s):=a^{n'}) \cdot Q\left(s,\left(\mathbf{a}^{-n}, a^{n'}\right)\right)$$}
\end{definition}
Algorithm \ref{alg:algorithm2} shows the process for certifying the robustness of the actions of each state while controlling the error. To correct the p-value in the multiple tests, we adapt the p-value for each test by multiplying it with the agent's importance factor (Line 4). Then we can perform the BH procedure (Line 5) to determine which tests should be rejected. Lastly, we obtain the set of certified agents $\mathcal{I}_{\boldsymbol{cert}}$ with certified bounds.

\begin{theorem}
 For each agent in $\mathcal{I}_{\boldsymbol{cert}}:=\left\{n \mid {d}_{n} \neq 0 \right\}$, the action can be certified as
 $\tilde{\pi}^n(z^n+\epsilon^n) = \tilde{\pi}^n(z^n)$, where $\|\epsilon^n\|_{2} \leq D:=min(d_n),\forall n \in \mathcal{I}_{\boldsymbol{cert}}.$
\end{theorem}

\begin{proof}
Considering each agent independently, given that agent $n$ updates its policy $\tilde{\pi}^n(z^n)$ in each state, under the condition $\mathbb{P}(\tilde{\pi}^n(z^n):=a^{m,n})> 0.5$, we can obtain the lower bound probability of selecting the $a^{m,n}$ and the upper-bound probability for the ``runner-up" action, $a^{r,n}$, for each agent and then compute the certified bound $d_n$. The minimum certified bound holds for any agent that satisfies the condition, denoted by the set $\mathcal{I}_{\boldsymbol{cert}}$.
\end{proof}
%\textit{Proof}.  Considering each agent independently, given that agent $n$ updates its policy $\tilde{\pi}^n(z^n)$ in each state, under the condition $\mathbb{P}(\tilde{\pi}^n(z^n):=a^{m,n})> 0.5$, we can obtain the lower bound probability of selecting the $a^{m,n}$ and the upper-bound probability for the ``runner-up" action, $a^{r,n}$, for each agent and then compute the certified bound $d_n$. The minimum certified bound holds for any agent that satisfies the condition, denoted by the set $\mathcal{I}_{\boldsymbol{cert}}$.

\section{Robustness Guarantee on Global Reward}\label{globalreward}

To certify the bound of global reward under the certified perturbation bound for each step, the CRSC is no longer applicable, as it cannot find the lower bound of global reward. Therefore, we propose a tree-search-based method to find the global lower bound of the team reward under the certified bound of perturbation. 

The insight of implementing the search tree is that, if we cannot certify the bound of perturbation at some time steps for some agent, we can take the second most frequent action, which will result in a new trajectory. Then we can explore the new trajectory by developing it as an expanded branch of the search tree, which may result in a lower global reward. Thus, the minimum reward can be determined as the certified lower bound of the global reward after exploring all trajectories. The main function for the tree-search-based method is presented in Algorithm \ref{alg:algorithm3}. As it shows, at first, we figure out all possible actions to formulate the action list to be explored using the function \textproc{GetNode}. Then we perform the \textproc{Search} function to expand the tree based on each action node. Once all new trajectories have been explored, we obtain the certified bound of perturbation and the minimum reward among all leaf nodes. We also apply pruning to control the size of the search tree, which requires the reward in the environments to be non-negative. When the cumulative reward of the current node has already reached the lower bound, it can be pruned, as the subsequent tree will not lead to a lower bound.
%\begin{figure}
%\centering
%\includegraphics[width=0.4\linewidth]{LaTeX/avrqmix.png}
%\includegraphics[width=0.4\linewidth]{LaTeX/avrvdn.png}
%\caption{Empirical average bounds with 95\% confidence interval of smoothed and original policies under different smoothing variance $\sigma$.}
%\label{Fig:aver2}
%\end{figure}

%%%%%%%%%%%%visit-3 - RWJ %%%%%%%%%%%%%%%%

\section{Experiments}

We first present the certified lower bound of the global reward under the certified perturbation, then we show the certified robustness for actions per state. Moreover, since our method is general and can be applied in single-agent systems, we will show the comparison experiments with the state-of-the-art baseline on certifying the global reward for RL with a single agent. 

\textbf{Baseline} In single agent environments, we compare our method with the state-of-the-art RL certification algorithm, CROP-LORE \cite{wu2021crop}. CROP-LORE is based on local policy smoothing that has a similar goal to our work -- obtaining a lower bound of the global reward under the certified bound for the actions of each state. Since CROP-LORE also employs the tree-search-based algorithm, we follow the same setting for a fair comparison. For certifying the c-MARLs, since there is no existing solution, we apply the PGD attack \cite{PGD} to demonstrate the validity of the certified bounds. 

\textbf{Environments} For the single agent environment , we use the ``Freeway" in OpenAI Gym \cite{openai}, which is the most stable game reported in the baseline. To demonstrate the performance of our method on c-MARLs, we choose two environments ``Checkers" with {\it two} agents and ``Switches" with {\it four} agents from ma-gym \cite{magym}. Details of the environments can be found in {\bf Appendix D}. Extra experiments in the environment -- ``Traffic Junction'' with four and ten agents can be found in {\bf Appendix E}.%Figure \ref{Fig:aver2} shows the benign performance of smoothed policy $\tilde{\pi}$ and original policy under different smoothing variances. 

\textbf{RL Algorithms}
We apply our method to certify the DQN trained by SA-MDP (PGD) and SA-MDP (CVX) \cite{zhang2020robust} in the single-agent setting since they have been empirically shown to achieve the highest certified robustness among all the baselines examined. 
For c-MARLs, we use VDN \cite{vdn} and QMIX \cite{qmix}, which are well-established value-based algorithms.

\textbf{Experiments setup} For all experiments, we sample noise 10,000 times for smoothing and set the discount factor $\gamma$ to 1.0. In the single-agent environment, we follow the same setting as the baseline, where the time step is 200 and the confidence level is $\alpha=0.05$. For c-MARLs, $\alpha=0.01$.
\begin{figure*}[htb]
\centering
\begin{subfigure}{1\linewidth}
    \centering
    \includegraphics[width=0.255\linewidth]{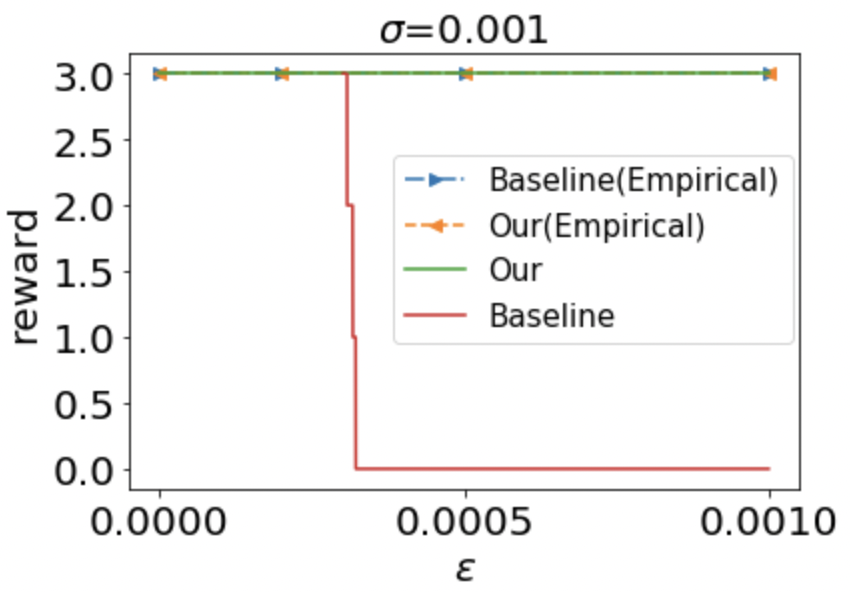}
    \includegraphics[width=0.24\linewidth]{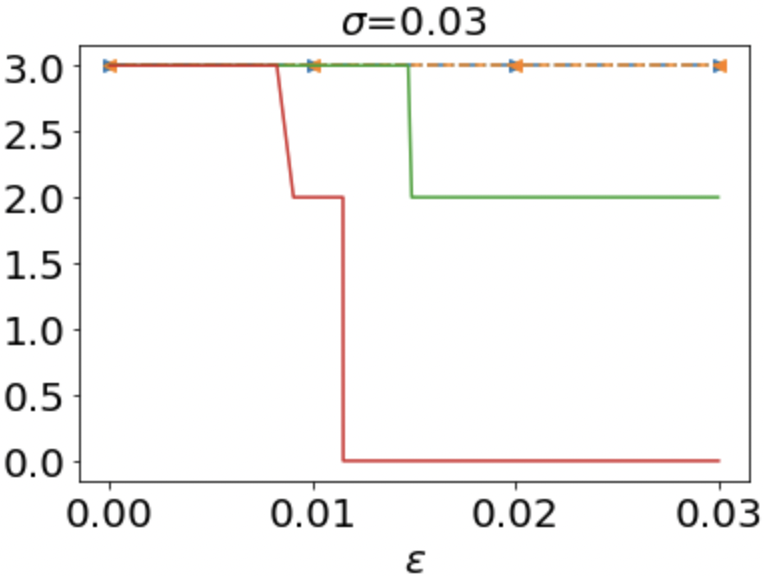}
    \includegraphics[width=0.24\linewidth]{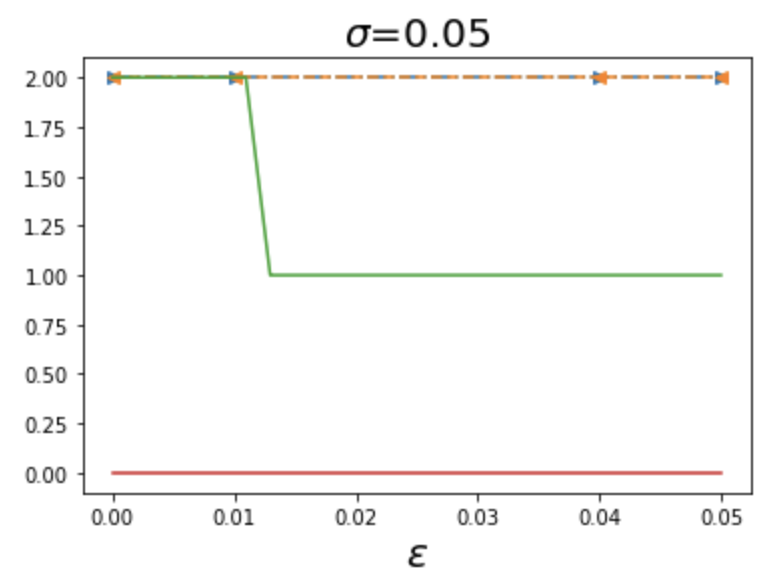}
    \includegraphics[width=0.24\linewidth]{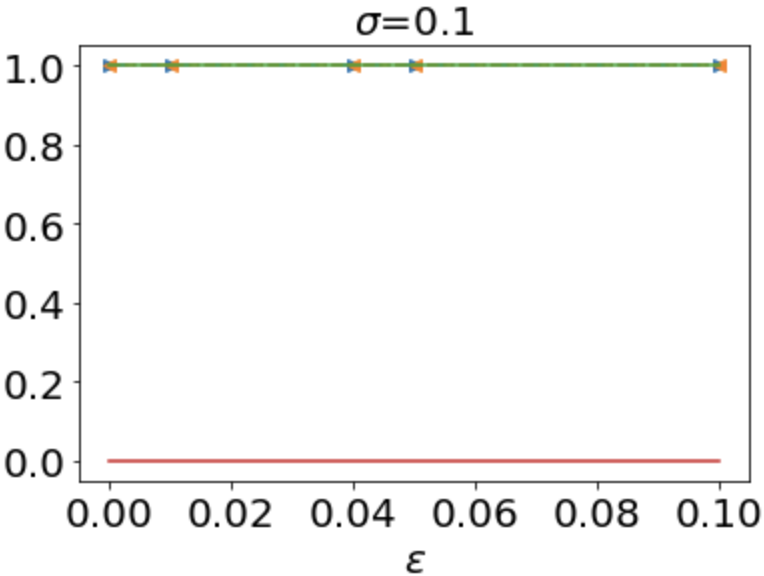}
    \vspace{-2mm}
    \caption{SA-MDP(CVX)}
     \label{fig:bs1}
 \end{subfigure}
 \begin{subfigure}{1\linewidth}
    \centering
    \includegraphics[width=0.255\linewidth]{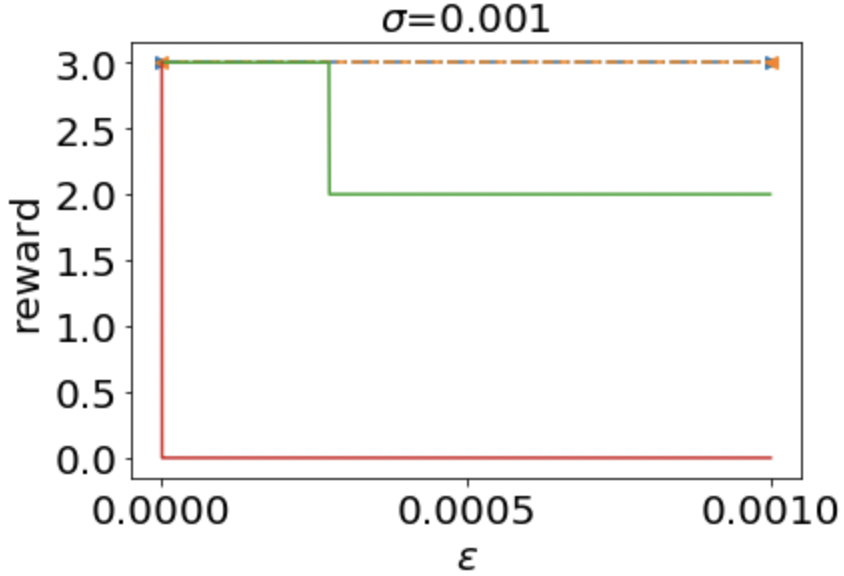}
    \includegraphics[width=0.24\linewidth]{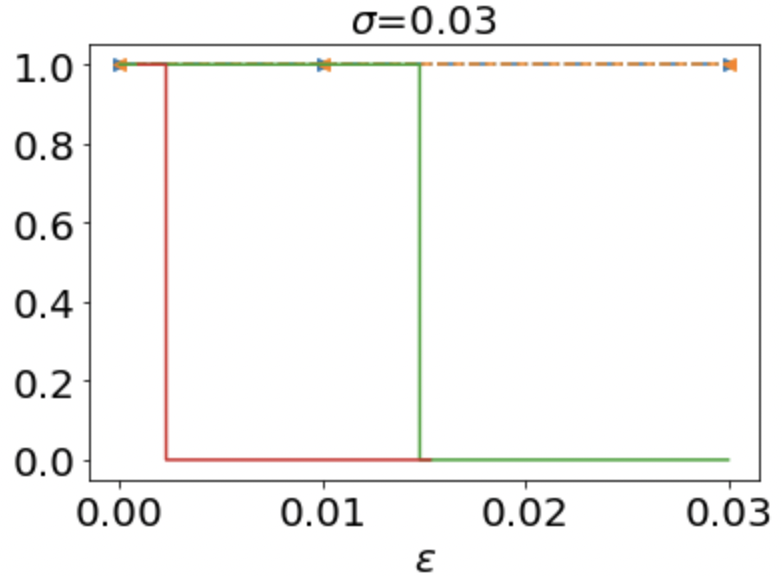}
    \includegraphics[width=0.24\linewidth]{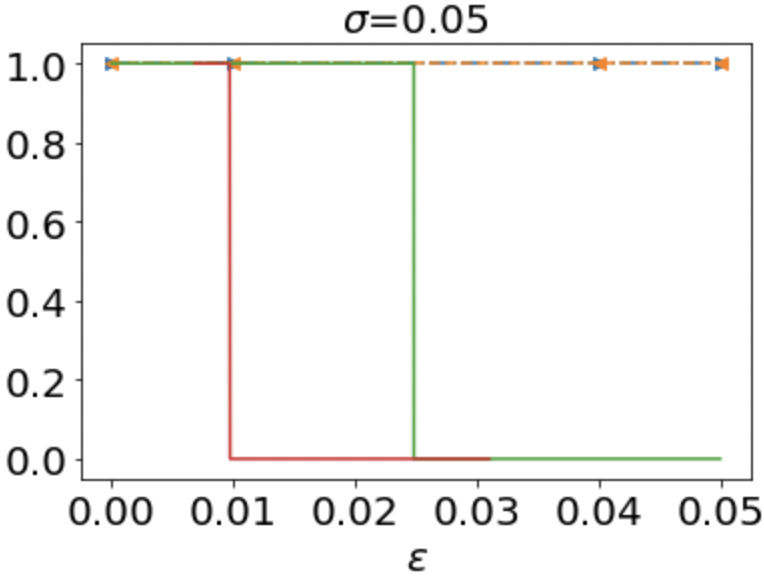}
    \includegraphics[width=0.24\linewidth]{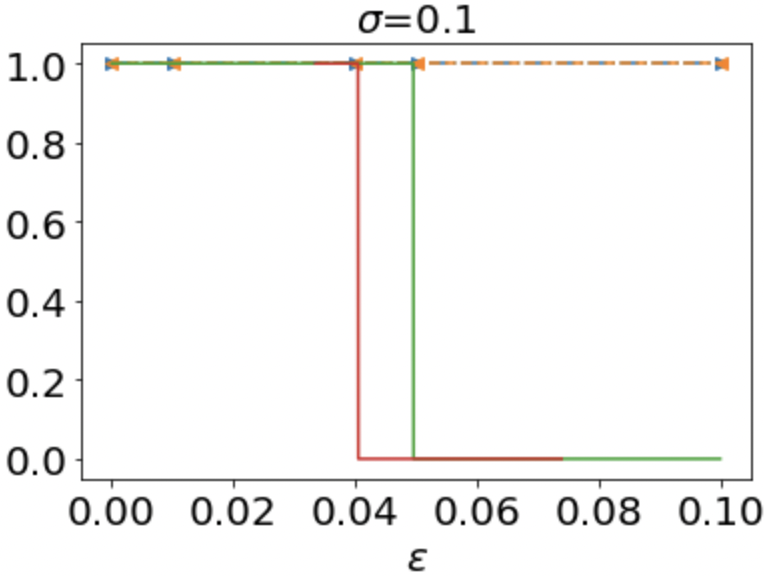}
        \vspace{-2mm}
    \caption{SA-MDP(PGD)}
     \label{fig:bs2}
     \vspace{-3mm}
 \end{subfigure}
\caption{Comparing the robustness certification of the total reward for SA-MDP in Freeway with \citet{wu2021crop}. Solid lines are the certified lower bounds of reward, and dashed lines indicate the empirical results under PGD attack.} 
\label{Fig:agentsbase}
\vspace{-3mm}
\end{figure*}
\begin{table*}[htb]
\begin{center}
\scalebox{0.8}{
\begin{tabular}{|c|c|c|c|c|c|c|c|c|c|c|c|}
\hline
{Models}&Game &No.agent&\multicolumn{3}{c|}{$\sigma=0.03$}& \multicolumn{3}{c|}{$\sigma=0.06$}& \multicolumn{3}{c|}{$\sigma=0.1$}   \\
\hline
 &&& $\epsilon_{cert}$& \multicolumn{2}{c|}{Reward}& $\epsilon_{cert}$& \multicolumn{2}{c|}{Reward}&$\epsilon_{cert}$& \multicolumn{2}{c|}{Reward}\\
 \hline
 &&&& Our&PGD& & Our&PGD&&Our&PGD\\
\hline
VDN&Checkers&2&0.0117 &79.84&79.84&0.0221 &79.84 &79.84 &0.0309 & 79.84&79.84  \\
\hline
QMIX&Checkers&2&0.0144 &19.96&19.96 &0.0369&19.96&19.96&0.0384 &19.96&19.96\\
\hline
VDN&Switch&4&0.0147 & 19.4&19.4 &0.284&14.4&19.4& 0.036&14&14.4 \\
\hline
QMIX *&Switch&4& 0.0173&-20&-20&0.0233&-20&-20 &0.038 &-20&-20  \\
\hline
\end{tabular}
}
\caption{Lower bound of global reward under the minimum certified bound of perturbation $\epsilon$, where the line with `*' denotes that we run the trajectory to the end without pruning to obtain the certified reward. }
\label{label:globalreward}
\end{center}
\vspace{-3mm}
\end{table*}

\subsection{Evaluate the robustness of the global reward}

\textbf{Compared with baseline on single agent} The baseline develops the smoothed policy based on the action-value function bounded by Lipschitz continuous, while our method is based on the probability of selecting the most frequent action. To make a fair comparison, we employ the same search tree structure as the baseline, which organises all possible trajectories and grows them by increasing the certified bound to choose an alternative action. The technical details are given in {\bf Appendix C}. 

As shown in Figure \ref{Fig:agentsbase}, our method obtains a tighter bound than the baseline. Since we measure the probability of selecting actions instead of the action value function to calculate the bound and choose action, we do not include the actions that have never been chosen in the possible action list, leading to a more reasonable action selection mechanism, resulting in a tighter calculated bound. Moreover, the Lipschitz continuity is used to compute the upper bound of the smoothed value function in the baseline, which is less tight than our bound based on high-probability guarantees. 

\textbf{Lower bound of global reward for c-MARLs} In Table \ref{label:globalreward}, we show the results of the lower bound of global reward under the minimum certified bound of perturbation $\epsilon_{cert}$. To perform pruning, the per-step reward in each environment are set to be non-negative. However, as the global reward obtained for QMIX on Switch are below zero, for this case, we run each trajectory to the end without pruning to calculate the global reward. We can see that VDN obtains higher reward compared to QMIX but is less robust (has lower $\epsilon_{cert}$). This is because during the training process, VDN simply adds rewards obtained by the two agents to achieve a centralisation, leading one agent to choosing a simpler strategy once another agent has learnt a useful strategy. On the other hand, QMIX employs a more complex network to centralise the agents instead of only adding their rewards, which helps the network to capture more complex interrelationships between different agents and encourage each one to learn. This leads to VDN achieving higher rewards faster than QMIX, but being more vulnerable to perturbations.

\begin{figure*}[htb]
\vspace{-3mm}
\centering
\begin{subfigure}{1\linewidth}
    \centering
    \includegraphics[width=0.252\linewidth]{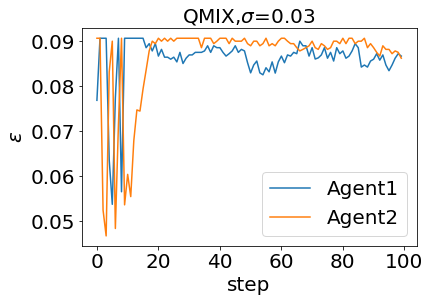}
    \includegraphics[width=0.244\linewidth]{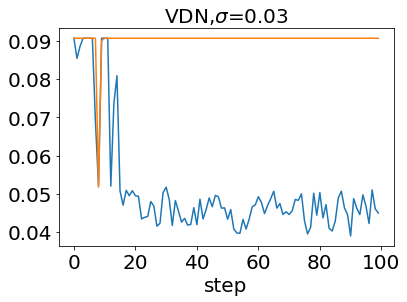}
    \includegraphics[width=0.244\linewidth]{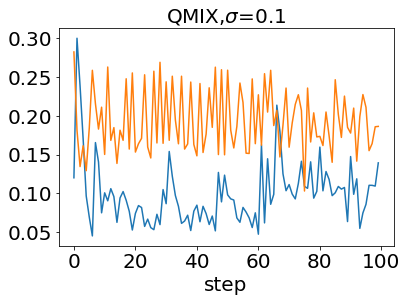}
    \includegraphics[width=0.244\linewidth]{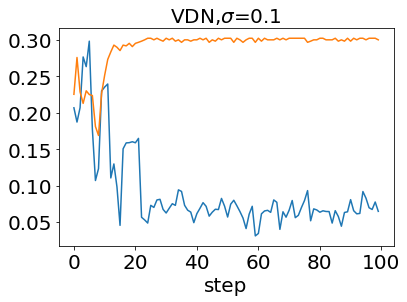}
    \vspace{-6mm}
    \caption{Certified bound of perturbation for actions of each state of each agent in Checkers}
    \label{fig:bound1}
 \end{subfigure}
 \begin{subfigure}{1\linewidth}
    \centering
        \vspace{2mm}
    \includegraphics[width=0.25\linewidth]{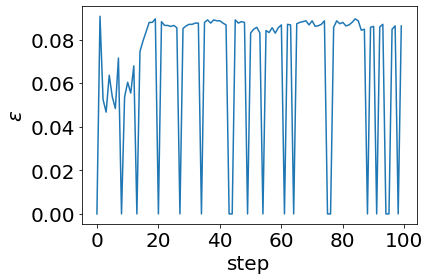}
\includegraphics[width=0.244\linewidth]{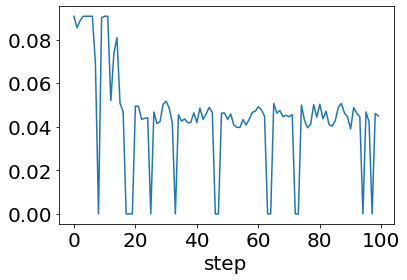}
\includegraphics[width=0.244\linewidth]{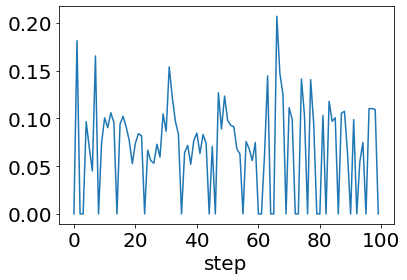}
\includegraphics[width=0.244\linewidth]{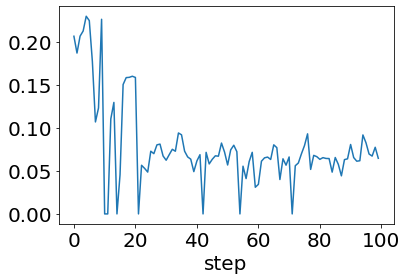}
\vspace{-6mm}
    \caption{Certified bound of perturbation for actions of each state in Checkers}
     \label{fig:bound2}
 \end{subfigure}
  \begin{subfigure}{1\linewidth}
    \centering
        \vspace{2mm}
    \includegraphics[width=0.25\linewidth]{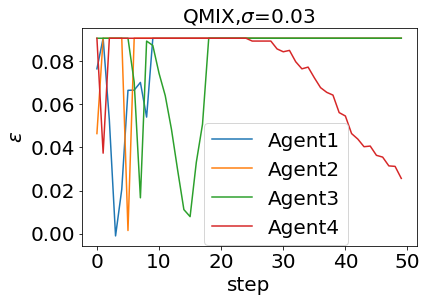}
    \includegraphics[width=0.244\linewidth]{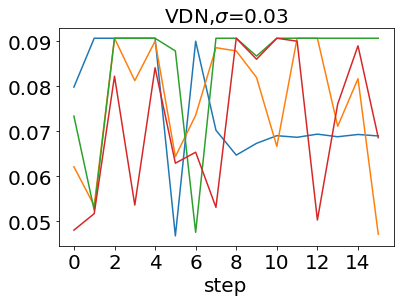}
    \includegraphics[width=0.244\linewidth]{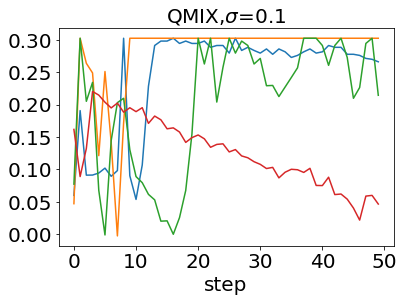}
    \includegraphics[width=0.244\linewidth]{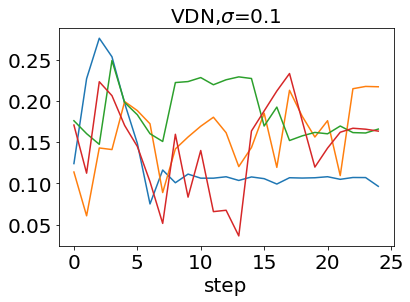}
    \vspace{-6mm}
    \caption{Certified bound of perturbation for per-state action of each agent in Switch.}
     \label{fig:bound3}
 \end{subfigure}
 \begin{subfigure}{1\linewidth}
    \centering
        \vspace{2mm}
    \includegraphics[width=0.252\linewidth]{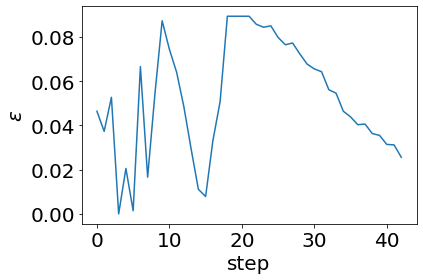}
\includegraphics[width=0.244\linewidth]{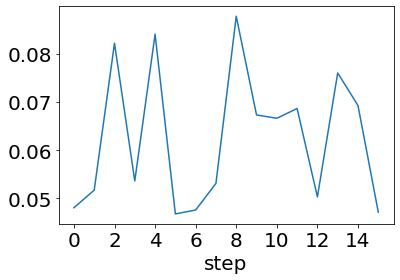}
\includegraphics[width=0.244\linewidth]{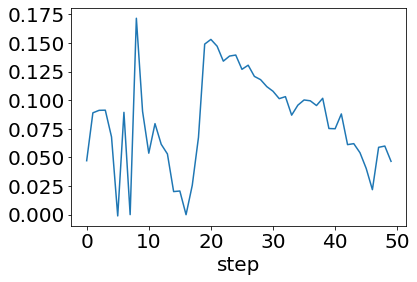}
\includegraphics[width=0.244\linewidth]{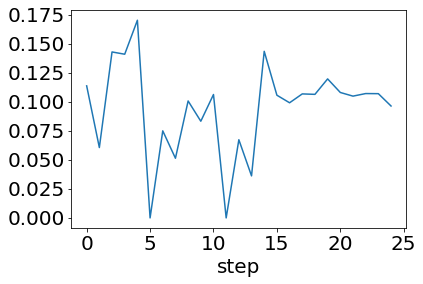}
\vspace{-6mm}
    \caption{Certified bound of perturbation for actions of each state in Switch}
     \label{fig:bound4}
     \vspace{-2mm}
 \end{subfigure}
\caption{Certified robustness for per-step action.}
\label{Fig:agents}
\vspace{-3mm}
\end{figure*}

\subsection{Evaluate the robustness for each state}

In Figure \ref{Fig:agents}, we present the certified perturbation bounded by $l_2$ norm for each agent and for all agents at each state separately. We see that in Checkers with two agents, the certified bound for each agent (trained by QMIX) is close to each other when the smoothing variance $\sigma$ is $0.03$. When we increase the variance to $0.1$, Agent2 engages a slightly higher bound than Agent1, which means that Agent2 is more robust. For the agents trained by VDN, Agent2 always has a much higher robustness bound than Agent1. It may be because when training QMIX, all agents are expected to learn useful strategy, while VDN only need some agents to learn well, and others may use lazier strategies, which results in
a big divergence in the robustness between agents of VDN. %while the robustness of agents in QMIX is similar. 
In Switch with four agents, we observe that, by applying our p-value corrected method, the locally certified bound at each step will not always take the minimum bound among all agents and ignore the bound of agents with low impact.

\section{Related work}
\textbf{Adversarial Attacks on DRLs} Existing attack solutions mainly focused on attacking single-agent RL systems, such as \citet{huang2017adversarial,lin2017tactics,kos2017delving,weng2019toward}. For attacking c-MARLs, there are notably two existing works. \citet{lin2020robustness} proposed to train a policy network to find a wrong action that the victim agent is expected to take and set it as the targeted adversarial example. \citet{attackcm2022} then proposed to craft a stronger adversary by using a model-based approach.

\textbf{Robustness Certification of DRLs} \citet{lutjens2020certified} first proposed a certified defence on the observations of DRLs. \citet{zhang2020robust} then provided empirically provable certificates to ensure that the action does not change at each state. However, this method cannot provide robustness certification for the reward if the action is changed under attacks. To tackle this problem, \citet{kumar2021policy} proposed to directly certify the total reward via randomised smoothing-based defence, but this method cannot achieve robustness certification at the action level. \citet{wu2021crop} then proposed a policy smoothing method based on the randomised smoothing of the action-value function, which is chosen as the baseline in this paper for certifying the robustness of global reward under a single-agent scenario. However, all existing methods can only work on single-agent systems. To the best of our knowledge, this paper is the first work to certify the robustness of cooperative multi-agent RL systems.

\section{Conclusion}
We propose the first robustness certification solution for c-MARLs. By combining the FDR-controlling strategy with the importance factor of each agent, we certify the actions for each state while mitigating the multiple testing problem. In addition, a tree-search-based algorithm is applied to obtain a lower bound of the global reward. Our method is also applicable to single-agent RL systems, where it can obtain tighter bound than the state-of-the-art certification methods. 

\section{Acknowledgements}
This work is supported by the Partnership Resource Fund of ORCA Hub via the UK EPSRC under project [EP/R026173/1]. Ronghui was funded by the Faculty of Science and Technology at Lancaster University. We thank the High-End Computing facility at Lancaster University for the computing resources. We also thank Matheus Alves and Peng Gao for proofreading. 

\balance
\bibliography{aaai23}

\end{document}